\documentclass[a4paper,11pt]{article}
\usepackage{geometry}
\usepackage{blindtext}

\usepackage[colorlinks=true, allcolors=blue]{hyperref}

\usepackage{amsmath}
\usepackage{amsthm}
\allowdisplaybreaks
\usepackage{amssymb}
\usepackage{amsfonts}
\usepackage{bm}
\usepackage{bbm}
\usepackage{thmtools, thm-restate}
\usepackage[utf8]{inputenc}
\usepackage[T1]{fontenc}

\usepackage{color} 
\usepackage{float}
\usepackage{algorithm}
\usepackage{algpseudocode}
\usepackage{enumitem}
\usepackage{comment}
\usepackage{framed}
\usepackage{dirtytalk}

\DeclareBoldMathCommand{\p}{p}
\DeclareBoldMathCommand{\z}{z}
\DeclareBoldMathCommand{\w}{w}
\DeclareBoldMathCommand{\W}{W}
\DeclareBoldMathCommand{\y}{y}
\DeclareBoldMathCommand{\h}{h}
\DeclareBoldMathCommand{\q}{q}
\DeclareBoldMathCommand{\e}{e}
\DeclareBoldMathCommand{\b}{b}
\DeclareBoldMathCommand{\g}{g}
\DeclareBoldMathCommand{\u}{u}
\DeclareBoldMathCommand{\U}{U}
\DeclareBoldMathCommand{\L}{L}
\DeclareBoldMathCommand{\0}{0}
\DeclareBoldMathCommand{\a}{a}
\DeclareBoldMathCommand{\x}{x}
\DeclareBoldMathCommand{\c}{c}
\DeclareBoldMathCommand{\I}{I}
\DeclareBoldMathCommand{\bell}{\ell}

\newcommand{\vs}{\v{s}}

\renewcommand{\hat}{\widehat}
\renewcommand{\tilde}{\widetilde}
\renewcommand{\epsilon}{\varepsilon}

\newcommand{\KL}{\textnormal{KL}}

\newcommand{\reals}{\mathbb{R}}

\newcommand{\ltil}{\tilde{\ell}}
\newcommand{\half}{\tfrac{1}{2}}
\newcommand{\id}{\mathbbm{1}}

\newcommand{\domainw}{\mathcal{W}}
\renewcommand{\ln}{\log}

\newcommand{\clip}{\textnormal{clip}}
\newcommand{\ball}{\Theta_b}
\newcommand{\ybar}{\overline{y}}

\newcommand{\sumT}{\sum_{t=1}^T}

\newcommand{\sumt}{\sum_{s=1}^t}

\newcommand{\Fset}{\mathcal{F}}
\newcommand{\Xset}{\mathcal{X}}
\newcommand{\Yset}{\mathcal{Y}}

\newcommand{\ymax}{l} \newcommand{\vaw}{\textnormal{VAW}}

\newcommand{\fbar}{\bar{f}}

\DeclareMathOperator{\E}{\mathbb E}

\DeclareMathOperator*{\argmin}{argmin}

\newtheorem{proposition}{Proposition}
\newtheorem{theorem}{Theorem}
\newtheorem{lemma}{Lemma}
\newtheorem{corollary}{Corollary}

\theoremstyle{definition}

\newtheorem{example}{Example}

\theoremstyle{remark}

\DeclareRobustCommand{\VAN}[3]{#2} 
\newcommand{\TODO}[1]{\ifmmode
\text{\textcolor{red}{TODO: #1}}
\else
\textcolor{red}{TODO: #1}
\fi
}

\newcommand{\Dirk}[1]{\ifmmode
\text{\textcolor{red}{Dirk: #1}}
\else
\textcolor{red}{Dirk: #1}
\fi
}

\DeclareRobustCommand{\VAN}[3]{#2}

\title{
High-Probability Risk Bounds via Sequential Predictors
}

\author{
  Dirk van der Hoeven\thanks{Korteweg-de Vries Institute for Mathematics, University of Amsterdam, The Netherlands, \href{mailto:dirk@dirkvanderhoeven.com}{dirk@dirkvanderhoeven.com}}
  \qquad
  Nikita Zhivotovskiy\thanks{Department of Statistics, University of California, Berkeley, USA, \href{mailto:zhivotovskiy@berkeley.edu}{zhivotovskiy@berkeley.edu}}
  \qquad
  Nicol{\`o} Cesa-Bianchi\thanks{Università degli Studi di Milano and Politecnico di Milano, Milano, Italy, \href{mailto:nicolo.cesa-bianchi@unimi.it}{nicolo.cesa-bianchi@unimi.it}}
}

\date{}
\begin{document}

\maketitle

\begin{abstract}
{Online learning methods yield sequential regret bounds under minimal assumptions and  provide in-expectation risk bounds for statistical learning.
However, despite the apparent advantage of online guarantees over their statistical counterparts, recent findings indicate that in 
many important cases,
regret bounds may not guarantee tight high-probability risk bounds in the statistical setting.
In this work we show that online to batch conversions applied to general online learning algorithms can bypass this limitation.
Via a general second-order correction to the loss function defining the regret, we obtain nearly optimal high-probability risk bounds for several classical statistical estimation problems, such as discrete distribution estimation, linear regression, logistic regression, and---more generally---conditional density estimation.}
Our analysis relies on the fact that many online learning algorithms are improper, as they are not restricted to 
 use predictors from a given reference class.
   The improper nature of our estimators enables significant improvements in the dependencies on various problem parameters. Finally, we discuss some computational advantages of our sequential algorithms over their existing batch counterparts.
\end{abstract}

\section{Introduction}\label{sec:intro}

One of the standard methods for the statistical analysis of learning algorithms is to exploit the corresponding risk (regret) bounds in the online learning setup, a technique known as \emph{online to batch conversion}. This idea has a long history in the statistics and machine learning literature. For instance, Vapnik and Chervonenkis \cite{vapnik74theory} used the online mistake bound of the Perceptron algorithm \cite{novikoff62convergence} to bound its expected risk in the batch statistical setting with i.i.d.\ data. Early works on kernel methods \cite{Aizerman1965} and early stopping criteria \cite[Theorem 4.1]{vapnik74theory} also used online to batch conversion arguments. {Over the years, sequential methods have shown numerous applications in the analysis of purely statistical problems, such as density estimation \cite{barron1987bayes, yang1999information, catoni1997mixture} and aggregation of estimators \cite{tsybakov2003optimal, juditsky2008learning, audibert2009fast}.
Summarizing the existing connections between sequential and statistical analysis, it is now well established that any \emph{regret bound} in the online learning setup can be translated into an in-expectation excess risk bound, provided that the loss is convex \cite[Theorem 5.1]{shalev2012online}.}

The situation is more subtle when we are interested in excess risk bounds that hold \emph{with high probability}, and even, remarkably, \emph{constant probability bounds}, as discussed further below. The work of Littlestone \cite{Littlestone1989} provided the optimal high-probability online to batch conversion in the realizable binary classification setup. When one is interested in the so-called slow rate $O\left(\frac{1}{\sqrt{T}}\right)$, where $T$ is the sample size, the optimal high-probability bounds typically follow from martingale extensions of standard concentration inequalities \cite{cesa2004generalization}. More recently, Kakade and Tewari \cite{kakade2009generalization} provided a high probability $O\left(\frac{1}{T}\right)$ excess risk bound for strongly convex and Lipschitz losses. Their results were further extended to the more general exp-concave losses by Mehta \cite{mehta2017fast}. See also \cite{puchkin23a}. The fundamental limitation of almost all the abovementioned $O\left(\frac{1}{T}\right)$ bounds is that the learning procedures are assumed to be \emph{proper}: These are learning algorithms that output their models in a particular reference class, usually assumed to be convex.

The importance of using \emph{improper} learning procedures, which are allowed to make predictions independently of any specific reference class, has been recently highlighted in several contexts in both statistical and online learning. The work of Foster et al. \cite{foster2018logistic}---see also the work of Kakade and Ng \cite{kakade2004online}---showed that in order to get $O\left(\frac{1}{T}\right)$ risk bounds for logistic regression, one should use improper learners to bypass the prohibitive exponential dependence on the parameters of the problem that appears for any proper learning procedure, as shown by Hazan, Koren, and Levy \cite{hazan2014logistic}. Similarly, Va\vs kevi\v{c}ius
 and Zhivotovskiy
\cite{vavskevivcius2020suboptimality}---see also \cite{forster2002relative, shamir2015sample}---showed the necessity of being improper in the context of batch linear regression with squared loss. Their analysis allows to completely ignore the dependence on the distribution of the random design matrix, which becomes impossible when restricted to proper learners only. Finally, from the standard perspective of model aggregation---see the work \cite{tsybakov2003optimal} of Tsybakov for the exact setup---using finite (and therefore non-convex) families of predictors, one should use \emph{improper estimators} to achieve the optimal $O\left(\frac{1}{T}\right)$ excess risk bound \cite[Section 3.5]{catoni2004statistical}, \cite{juditsky2008learning}.

When working with \emph{improper learners}, converting a constant (logarithmic) online learning regret bound into a $O\left(\frac{1}{T}\right)$ high-probability excess risk bound is a challenging problem. A curious result of Audibert \cite{audibert2007progressive} shows that the standard exponential weights algorithm, while giving an optimal $O\left(\frac{1}{T}\right)$ excess risk bound in-expectation for the squared loss, does not do so with high probability. In fact, the author of \cite{audibert2007progressive} showed an $\Omega\left(\frac{1}{\sqrt{T}}\right)$ lower bound for online to batch converted exponential weights in the high-probability setup. This behaviour is due to the improper structure of exponential weights: when making a linear combination of a finite set of predictors, one can output a predictor that outperforms the best predictor in the finite set. The negative part of the excess risk can compensate its positive part, so that the excess risk remains small in expectation, while it can become large in probability. The main message of this article is to further highlight the following principle, which will be explained in more detail below.
\begin{framed}
While in-expectation online to batch conversions are widely used in the statistics and machine learning literature, we argue that they should be employed with caution when dealing with improper learners. This is because, in some cases, it is hard to establish a non-trivial constant probability excess risk upper bound, and standard confidence boosting methods may not be applicable. Therefore, we need to explore different approaches to the design of online algorithms with small high-probability excess risk bounds.
\end{framed}
The route to obtain high probability $O\left(\frac{1}{T}\right)$ excess risk bounds via online to batch conversions of improper learners was initiated by Wintenberger \cite{wintenberger2017optimal}---see also \cite{gaillard2018efficient}. Improving upon their results, the authors of this article derived in \cite{van2022regret} a simple high-probability analysis for strongly convex losses that covers more general setups (including linear regression) and showed multiple applications of the negative terms appearing in the analysis of online learning algorithms. In this work we provide a new analysis that extends to exp-concave losses while focusing on explicit and improved dependencies on different parameters of the learning problems. In contrast to \cite{van2022regret}, here we directly work with losses rather than working with loss gradients. As a consequence, we only assume that the loss is bounded instead of being Lipschitz. 
One particular application of our ideas is in logistic regression, where the authors of \cite{mourtada2019improper} noted that the online to batch conversion in \cite{foster2018logistic} based on the confidence boosting scheme is incorrect. The problem is in the improper nature of their algorithm: a good in-expectation performance of an improper algorithm does not necessarily lead to a good performance, even with constant probability. So, it remained open whether one could construct an online to batch conversion achieving a logarithmic dependence on the parameters with high probability, as originally claimed in \cite{foster2018logistic}. In this work, among other results, we provide such a conversion. 

The remainder of the paper is structured as follows. In Section~\ref{sec:prelim} we formally introduce the setting and prove some inequalities we use throughout the paper. In Section~\ref{sec:mainresult} we provide the main technical result of this work: any algorithm with regret $R_T$ for arbitrary $\alpha$-exp concave losses with absolute differences bounded my $m$ can be modified to guarantee an excess risk of order $\frac{1}{T}\big(R_T + \gamma \log\frac{1}{\delta}\big)$ with probability at least $1- \delta$, where $\gamma = 4\max\{\frac{1}{\alpha}, m\}$. We apply the main result to conditional density estimation (Section~\ref{sec:density}), logistic regression (Section~\ref{sec:logreg}), and generalized linear models (Section~\ref{sec:glm}). In Section~\ref{sec:modelselagg} we show that a simple modification of exponential weights can be used to derive the optimal rate for model aggregation. Finally, in Section~\ref{sec:linleg} we apply our results to linear regression with squared loss, and derive optimal rates up to log factors with a computationally efficient algorithm.

\section{Notation and preliminaries}\label{sec:prelim}

We assume that we are given a family $\mathcal F$ of real-valued functions defined on a measurable instance space $\mathcal X$. We observe $T$ i.i.d.\ observations $(X_t, Y_t)_{t = 1}^T$ distributed according to some unknown distribution $\mathbb{P}$ on $\mathcal X \times \mathbb{R}$. Throughout the paper, we use the notation $\E_{t-1}[\,\cdot\,] = \E\big[\,\cdot\,|\,(Y_1, X_1), \ldots, (Y_{t-1}, X_{t-1})\big]$. Given a loss function $\ell: \mathbb{R}^2 \to \mathbb{R}$, the \emph{risk} of $f: \mathcal X \to \mathbb{R}$ is given by
$
\E_{X, Y}\ell\big(f(X), Y\big),
$
where the expectation is taken with respect to the joint distribution $\mathbb{P}$ of $X$ and $Y$. We are interested in bounding the \emph{excess risk}
$$
\E_{X, Y}\ell\big(\hat f(X), Y\big) - \inf\limits_{f \in \mathcal F}\E_{X, Y}\ell\big(f(X), Y\big)~,
$$
where $\hat f$ is constructed based on the sample $(X_t, Y_t)_{t = 1}^T$. When a particular loss is clear from the context, we sometimes use the abbreviated notation 
\[
R(f) = \E\ell\big(\hat f(X), Y\big)~.
\]
One of the key assumptions on the loss we use is the \emph{exp-concavity}.
For $\domainw \subseteq \reals$ we say that a function $h:\domainw \to \mathbb{R}$ is $\alpha$-exp-concave if
\begin{align}\label{eq:defexpconcave}
    \alpha h'(w)^2 \leq h''(w) \qquad \text{for all $w \in \domainw$}~.
\end{align}
Here $h'$ and $h''$ denote the first and the second derivatives of $h$ respectively.
We say that the loss function $\ell(\cdot, y)$ is $\alpha$-exp-concave if it is an $\alpha$-exp-concave function with respect to its first argument for all $y$ in the domain of $Y$. The analysis of these losses traces back to the foundational work by Vovk \cite{vovk1990aggregating}. A more detailed treatment of these losses appears in the monograph \cite[Section 3.3]{cesa2006prediction}. 
We now prove the following simple lemma\footnote{The result of Lemma~\ref{lem:negterm} also appears explicitly in recent work \cite{saad2022constant}, although with a different choice of $\gamma$. Their result is used in a different context.} which will play an important role in our derivations. 
\begin{lemma}\label{lem:negterm}
Consider an $\alpha$-exp-concave function $h:\domainw \to \mathbb{R}$ satisfying $h(x) - h(y) \le m$ for all $x, y \in \mathcal W$, where $m > 0$. Let $\gamma = 4 \max\big\{m, \frac{1}{\alpha}\big\}$. Then,
\begin{align*}
    h\big( \half x + \half y\big) \leq \half h(x) + \half h(y) -  \frac{\big(h(x) - h(y)\big)^2}{4 \gamma}~, \qquad\text{for all $x, y \in \domainw$}~.
\end{align*}
\end{lemma}
\begin{proof}
Fix any $z \in \mathcal W$, and let $g(\cdot) = h(\cdot) - h(z) - \frac{(h(\cdot) - h(z))^2}{\gamma}$. Note that $g$ is convex because 
\begin{align*}
    g''(x) &= h''(x) - \frac{2}{\gamma}\big(h''(x) (h(x) - h(z)) + (h'(x))^2\big) \\
    &\geq  h''(x) - \frac{2}{\gamma}\left(h''(x)m + \frac{ h''(x)}{\alpha}\right) \ge 0~,
\end{align*}
where in the first inequality we used the definition \eqref{eq:defexpconcave} of exp-concavity, which implies, in particular, that $h''(x) \ge 0$, and the assumption $h(x) - h(z) \le m$, while in the second inequality we used the definition of $\gamma$.
For $x, y \in \domainw$, the convexity of $g$ implies $g\big(\half x + \half y\big) \leq \half g(x) + \half g(y)$. When reordered, this gives
\[
    h\big(\half x + \half y\big) \leq \half h(x) + \half h(y) -  \frac{(h(x) - h(z))^2}{2 \gamma} - \frac{(h(y) - h(z))^2}{2 \gamma} + \frac{\left(h\big(\half x + \half y\big) - h(z)\right)^2}{\gamma}~.
\]
Assume without the loss of generality that $h(x) \ge h(y)$.
Consider two cases. If $h\big(\half x + \half y\big) \le h(y)$, then choose $z = \half x + \half y$. In this case,
\begin{align*}
h\big(\half x + \half y\big) &\leq \half h(x) + \half h(y) -  \frac{\left(h(x) - h\big(\half x + \half y\big)\right)^2}{2 \gamma} - \frac{\left(h(y) - h\big(\half x + \half y\big)\right)^2}{2 \gamma}
\\
&\leq \half h(x) + \half h(y) -  \frac{(h(x) - h(y))^2}{2 \gamma}~.
\end{align*}
Otherwise, if $h\big(\half x + \half y\big) > h(y)$, we choose $z = y$. In this case, using the convexity of $h$, we have
\begin{align*}
    h\big(\half x + \half y\big) &\leq \half h(x) + \half h(y) -  \frac{(h(x) - h(y))^2}{2 \gamma}  + \frac{\left(h\big(\half x + \half y\big) - h(y)\right)^2}{\gamma}
    \\
    &\leq \half h(x) + \half h(y) -  \frac{(h(x) - h(y))^2}{2 \gamma}  + \frac{\left(\half h(x) - \half h(y)\right)^2}{\gamma}
    \\
    &= \half h(x) + \half h(y) -  \frac{(h(x) - h(y))^2}{4 \gamma}~.
\end{align*}
\\The claim follows.
\end{proof}
We note that the assumption $h(x) - h(z) \le m$ is always satisfied if $h(\cdot)$ takes its values in $[0, m]$. However, in our application to logarithmic loss, it will be easier to control $h(x) - h(z) \le m$ without assuming that $h(\cdot)$ itself is bounded by $m$.
A simple rearrangement of the inequality proven in Lemma~\ref{lem:negterm} shows that, for any $\alpha$-exp-concave function satisfying the assumptions of Lemma~\ref{lem:negterm},
\begin{align}\label{eq:simplenegative}
    h(x) - h(y) \leq 2 h(\half x + \half x) - 2 h\big(\half x + \half y\big) - \frac{1}{2\gamma}\big(h(x) - h(y)\big)^2 \qquad\text{for all $x, y \in \domainw$}~,
\end{align}
where $\gamma = 4 \max\big\{m, \frac{1}{\alpha}\big\}$. In particular, the negative quadratic term in~\eqref{eq:simplenegative} is what compensates for the variance of the online to batch conversion. In the following section, we show precisely how. 

\section{Online to batch for improper learners with high probability}\label{sec:mainresult}

In this section, we state our main technical result. 
Let $\hat f_1, \ldots, \hat f_T$ be the sequence of predictors obtained by running some online algorithm on $(X_t, Y_t)_{t = 1}^T$. Here we mean online algorithm in the sense that each $\hat f_k$ only depends on $(X_t, Y_t)_{t = 1}^{k - 1}$. Note that we do not insist on $\hat f_t \in \Fset$, so these predictors may be improper.
Fix an $\alpha$-exp concave loss function $\ell$.
Because each $\hat f_k$ only depends on $(X_t, Y_t)_{t = 1}^{k - 1}$, we may consider $\hat f_1, \ldots, \hat f_T$ obtained by running our online algorithm (to be chosen later) on the shifted online loss function
\begin{align}\label{eq:shiftloss}
    \tilde{\ell}_t(f) = \ell\big(\half f(X_t) + \half \hat{f}_t(X_t), Y_t\big)~,
\end{align}
which is also $\alpha$-exp concave (Lemma~\ref{lem:surrogateexpconcave} in Appendix). We say that $\hat f_1, \ldots, \hat f_T$ satisfy the \emph{bounded shifted regret} condition if
\begin{equation}\label{eq:regretreq}
\sumT\Big(\tilde{\ell}_t\big(\hat f_{t}\big) -  \E_{f \sim Q}[\tilde{\ell}_t(f)]\Big)\le R_T
\end{equation}
almost surely with respect to $(X_t, Y_t)_{t = 1}^T$, where $Q$ is some fixed distribution over $\Fset$ and $\ltil_t$ is defined in~\eqref{eq:shiftloss}. We now show that the excess risk of $\frac{1}{T}\big(\hat f_1 + \cdots + \hat f_T\big)$, which is the standard predictor for online to batch conversions, is bounded with high probability in terms of the shifted regret.
\begin{theorem}\label{th:improperOTB}
Suppose that the loss function $\ell: \domainw \times \Yset \mapsto \mathbb{R}$ is $\alpha$-exp concave in its first argument. Assume that $\hat f_1, \ldots, \hat f_T$ satisfy the bounded shifted regret condition \eqref{eq:regretreq}, and that additionally $\left|\ell(\hat{f}_t(X_t), Y_t) - \ell(f(X_t), Y_t)\right| \le m$ almost surely for all $t = 1, \ldots, T$ and $f \in \mathcal F$. Then, the risk of the averaged estimator
\begin{align}\label{eq:OTB}
    \fbar_T = \frac{1}{T}\sum_{t=1}^T \hat{f}_t
\end{align}
satisfies, with probability at least $1-\delta$ with respect to the random draw of $(X_t, Y_t)_{t = 1}^T$,
\begin{align*}
    R(\bar{f}_T) - \E_{f \sim Q}[R(f)] & \leq \frac{2 R_T + 2 \gamma \log(1/\delta)}{T},
\end{align*}
where $\gamma = 4 \max\big\{m, \frac{1}{\alpha}\big\}$.
\end{theorem}
\begin{proof}
Let $r_t = \ell(\hat{f}_t(X_t), Y_t) - \E_{f \sim Q}\big[\ell(f(X_t), Y_t)\big]$.
We start with an application of Jensen's inequality
\begin{align*}
    R(\bar{f}_T) - \E_{f \sim Q}[R(f)]
& \leq
    \frac{1}{T}\sumT\E_{t-1}\Big[\ell(\hat{f}_t(X_t), Y_t) - \E_{f \sim Q}\big[\ell(f(X_t), Y_t)\big]\Big]
\\& =
    \frac{1}{T}\sumT\E_{t-1}[r_t]
\\& =
    \frac{1}{T}\left(\sumT\big(\E_{t-1}[r_t + v_t] - (r_t + v_t)\big) + \sumT r_t + \sumT \big(v_t - \E_{t-1}[v_t]\big)\right),
\end{align*}
where the $v_t$ are arbitrary random variables.
Note that $|r_t| \le m$ due to our assumption. Choosing $v_t = \frac{r_t^2}{2\gamma}$ and using the definition of $\gamma$,
\begin{align*}
    v_t = \frac{r_t^2}{2\gamma}
\leq
    \frac{|r_t|m}{2\gamma}
\leq
    \frac{|r_t|}{8}~.
\end{align*}
Therefore,
\[
|r_t + v_t| \leq |r_t| + v_t \le \frac{9}{8} |r_t| \leq \frac{9}{8} m~.
\]
We now apply Lemma~\ref{lem:bernie} with $X_t =\E_{t-1}[r_t + v_t] - (r_t + v_t)$ observing that $\E_{t-1}[X_t] = 0$ and $|X_t| \le \frac{9}{4}m$. Therefore, for any $\lambda \in \big(0,\frac{4}{9m}\big]$ we have that, with probability at least $1 - \delta$,
\begin{align*}
    \sumT\big(\E_{t-1}[r_t + v_t] - (r_t + v_t)\big)
& \leq
    \lambda(e-2)\sumT \E_{t-1}\Big[\big(\E_{t-1}[r_t + v_t] - (r_t + v_t)\big)^2\Big] + \frac{1}{\lambda}\ln\frac{1}{\delta}
\\& \leq
    \lambda(e-2)\sumT \E_{t-1}\big[(r_t + v_t)^2\big] + \frac{1}{\lambda}\ln\frac{1}{\delta}
\\& \leq
    \lambda \sumT \E_{t-1}\big[r_t^2\big] + \frac{1}{\lambda}\ln\frac{1}{\delta}~,
    \end{align*}
where in the last inequality we used $|r_t + v_t| \leq \tfrac{9}{8} |r_t|$ and $(e-2)\big(\frac{9}{8}\big)^2 \leq 1$. Choosing $\lambda = \frac{1}{2\gamma} \leq \frac{4}{9m}$ and recalling our choice of $v_t$,
\begin{align*}
    \sumT\big(\E_{t-1}[r_t + v_t] - (r_t + v_t)\big) & \leq \sumT \E_{t-1}[v_t] +  2 \gamma \log\frac{1}{\delta}~,
\end{align*}
with probability at least $1 - \delta$. Therefore, again with probability at least $1 - \delta$,
\begin{align}\label{eq:tempmainproof}
    R(\bar{f}_T) - \E_{f \sim Q}[R(f)] & \leq \frac{2 \gamma}{T}\log\frac{1}{\delta} + \frac{1}{T}\sumT  (r_t + v_t)~.
\end{align}
By~\eqref{eq:simplenegative}, we obtain
\begin{align*}
    \sumT r_t & = \sumT \E_{f \sim Q}\Big[\ell\big(\hat{f}_t(X_t), Y_t\big) - \ell(f(X_t), Y_t)\Big] \\
    & \leq 2 \sumT \E_{f \sim Q}\bigg[\ell\big(\half\hat{f}_t(X_t) + \half\hat{f}_t(X_t), Y_t\big) - \ell\big(\half f(X_t) + \half \hat{f}_t(X_t), Y_t\big)\bigg]
    \\
    &\qquad - \frac{1}{2\gamma}\sumT\E_{f \sim Q}\left[\big(\ell(\hat{f}_t(X_t), Y_t) - \ell(f(X_t), Y_t)\big)^2\right]\\
    & \leq 2 R_T - \sumT v_t~,
\end{align*}
where the last inequality is by assumption~\eqref{eq:regretreq} on $\hat{f}_1, \ldots, \hat{f}_T$ and Jensen's inequality. Using~\eqref{eq:tempmainproof} we have that, with probability $1 - \delta$,
\begin{align*}
    R(\bar{f}_T) - \E_{f \sim Q}[R(f)] & \leq \frac{2 R_T + 2 \gamma \log(1/\delta)}{T}~,
\end{align*}
thus completing the proof.
\end{proof}
We now turn to analyzing the bound on the shifted regret. It appears that whenever we can guarantee that the standard regret (i.e., the difference between the cumulative loss $\ell_t$ of our algorithm minus the cumulative loss of any comparator $f$) is small, we can also show that the shifted regret is also small. To see that $R_T$ can indeed be small, observe that since $\tilde{\ell}_t$ are $\alpha$-exp concave (Lemma~\ref{lem:surrogateexpconcave}) we may use the Exponential Weights Algorithm (EWA) originally introduced in \cite{vovk1990aggregating, littlestone1994weighted} on the shifted loss $\tilde{\ell}_t$ to obtain $R_T \leq \frac{1}{\alpha}\KL(Q_1\|P_1)$, where $P_1$ is the prior EWA distribution over $\Fset$ and $\KL$ is the Kullback-Leibler divergence (see Appendix~\ref{app:exponential weights} for details on EWA). This is made formal by the following result. 
\begin{proposition}
\label{prop:boundingtheregret}
Suppose that the loss function $\ell: \domainw \times \Yset \mapsto \mathbb{R}$ is $\alpha$-exp concave in its first argument. Then, exponential weights algorithm on the sequence of losses $(\tilde{\ell}_t)_{t = 1}^T$ defined in equation~\eqref{eq:shiftloss} with prior $P_1$ guarantees that \[
R_T \leq \tfrac{1}{\alpha}\KL(Q\|P_1)~.
\]
\end{proposition}
The proof follows from Lemma~\ref{lem:EWshiftloss} in Appendix~\ref{app:exponential weights}.
A result similar to a combination of Theorem \ref{th:improperOTB} and Proposition \ref{prop:boundingtheregret} is known for general exp-concave losses, but only in expectation \cite[Corollary 4.1]{audibert2009fast}.  In the referenced paper, EWA is used with the losses $\ell_t$, but not their shifted counterparts $\tilde{\ell}_t$ as we do in Theorem \ref{th:improperOTB}. As we mentioned, EWA on the losses $\ell_t$ does not achieve the bound of Theorem \ref{th:improperOTB} as shown in \cite{audibert2007progressive}: The so-called \emph{progressive mixture rules} only imply a $O\big(\frac{1}{T}\big)$ excess risk bound in expectation, but not with high probability.

The idea of exploiting the curvature of the loss by using the \emph{midpoint prediction} $\half f(X_t) + \half \hat{f}_t(X_t)$ as in \eqref{eq:shiftloss} appeared earlier in the literature. In particular, a similar idea was used in \cite{mendelson2017aggregation, mendelson2019unrestricted, mourtada2021distribution} in the context of aggregation of heavy-tailed functions, as well as in \cite{bousquet2021fast, puchkin21a} in the context of classification with abstention. More recently, the same idea was used in \cite{saad2022constant} in the context of online learning with limited advice.

\paragraph{Technical overview of the results.}
We present a concise overview of essential technical ideas used in this paper. The cornerstone of our work lies in the synergy between Theorem \ref{th:improperOTB} and Proposition \ref{prop:boundingtheregret} with related results from online learning. We incorporate additional concepts tailored to specific applications. First, we apply application-specific prior distributions $P_1$ in Proposition \ref{prop:boundingtheregret}, encompassing uniform, Gaussian, and Dirichlet distributions. In our density estimation applications, we leverage adaptive truncation operators to prove nearly optimal high-probability excess risk bounds for improper estimators. In Section \ref{sec:discretedistrib}, we apply the \emph{suffix averaging} idea, recently employed in various contexts \cite{rakhlin2012making, harvey2019tight,aden2023optimal}, thereby achieving a high-probability bound on the Kullback-Leibler divergence in the estimation problem of discrete distributions supported on $d$ points. This shows a scaling rate of $O(\frac{d}{T})$, superior to the best possible rate $O(\frac{d\log(T)}{T})$ attained by conventional online algorithms. 
\paragraph{Additional notation.}
For a pair of functions $f, g$ defined on some common domain, we write $f \lesssim g$ (or $g \gtrsim f$) if there is a constant $c> 0$ such that for all $x$ in this domain it holds that $f(x) \le cg(x)$. Although we focus on explicit non-asymptotic results, we sometimes use the asymptotic $O(\cdot)$ and $\Omega(\cdot)$ notations to illustrate our bounds. The symbol $I$ denotes the identity matrix whose size is clear from the context. Depending on the context, we sometimes abuse the notation and write $\log(x)$ to denote $\log\big(\max\{x, 1\}\big)$, where $\log(x)$ refers to the natural logarithm.

\section{Density estimation under the logarithmic loss}\label{sec:density}

We first consider the general problem of density estimation. Namely, we are interested in the setup where given a sample $Z_1, \ldots, Z_T$ of independent copies of some random variable $Z$, we want to minimize the risk with respect to the \emph{logarithmic loss}. Given a density function $g(\cdot)$, this risk is defined as
\[
R(g) = \E_Z\big[-\log(g(Z))\big].
\]
We consider a reference class of densities $p(Z|\theta)$, parameterized by $\theta$ that belongs to some set $\Theta \subset \mathbb{R}^d$. For any estimator of the density $\hat{p}$ (not necessarily in the reference class) constructed based on the sample $Z_1, \ldots, Z_T$, we can define the excess risk with respect to logarithmic loss as
\begin{equation}
\label{eq:equationforexcessrisk}
    \E_Z\big[-\log(\hat{p}(Z))\big] - \inf\limits_{\theta \in \Theta}\E_Z\big[-\log(p(Z|\theta))\big]~.
\end{equation}
In the \emph{well-specified} case, one assumes that there there is $\theta^{\star}\in \Theta$ such that the density of $Z$ is $p(\cdot|\theta^{\star})$. In this case the excess risk has a particularly simple form, as it is easy to show that
\begin{equation*}
\label{eq:}
\E_Z\big[-\log(\hat{p}(Z))\big] - \inf\limits_{\theta \in \Theta}\E_Z\big[-\log(p(Z|\theta))\big] = \KL\big({p(\cdot|\theta^{\star})\|\hat{p}(\cdot)}\big)~,
\end{equation*}
and is thus non-negative. Here $\KL\big({p(\cdot|\theta^{\star})\|\hat{p}(\cdot)}\big)$ stands for the Kullback-Leibler divergence between the distributions induced by the densities  $p(\cdot|\theta^{\star})$ and $\hat{p}(\cdot)$ respectively. Our focus is on the general \emph{misspecified} case, where the excess risk \eqref{eq:equationforexcessrisk} can possibly be negative. 

Instead of attempting a survey of the vast statistical literature on density estimation, we only mention the key results where online algorithms are used to control the predictive risk with the logarithmic loss. The key contributions here are due to Barron and Yang \cite{barron1987bayes, yang1999information, yang2000mixing} and, independently, to Catoni \cite{catoni1997mixture, catoni2004statistical}. To upper bound the predictive risk in density estimation, these authors pioneered the application of the \emph{progressive mixture rule}, which in our notation is essentially the output of the standard EWA algorithm (with respect to the log-loss) averaged over $t = 1, \ldots, T$ as in \eqref{eq:OTB}. 
Subsequent papers on density estimation using similar online to batch conversions include \cite{juditsky2008learning, audibert2009fast}. See also the papers and the recent monograph of Zhang \cite{zhang2006zepsilon, zhang2006information, zhang2023mathematical}. Recent interest in these questions was sparked by the aforementioned work of Foster et al. \cite{foster2018logistic}, where the special case of logistic regression is analyzed. We additionally refer to \cite{mourtada2019improper, bilodeau2021minimax} for a detailed survey of related results. All the abovementioned results involving progressive mixture rules suffer from the problem observed by Audibert \cite{audibert2007progressive}: the EWA algorithm does not imply high-probability excess risk bounds in the misspecified case. The remainder of the section is devoted to providing sharp high-probability bounds on the excess risk with respect to the logarithmic loss.

\subsection{Conditional density estimation}
In this section, we focus on \emph{conditional density estimation}. In this setup, 
a density over outcomes $y \in \Yset \subseteq \reals$ given inputs {$x \in \Xset$} and $\theta \in \Theta \subseteq \reals^d$ is denoted by {$p(y|x, \theta)$}. In Subsection~\ref{sec:glm} we analyze the special case of \emph{generalized linear models}, whose density can be written as ${p(\cdot|x^\top\theta)}$. 

The goal of density estimation is to control the log-loss excess risk, which, for some distribution $Q$ over $\Theta$, is defined as \begin{align*}
    \E\big[-\log(\bar{p}(Y|X))\big] - \E\Big[\E_{\theta \sim Q}\big[-\log({p}(Y|X, \theta))\big]\Big]~,
\end{align*}
{where the expectation is taken with respect to the pair $(X, Y)$ and $\bar{p}$ denotes our estimator.} 
Since $-\log(\cdot)$ is a $1$-exp-concave loss function, Theorem~\ref{th:improperOTB} should give a high-probability result. However, since $-\log(\cdot)$ is an unbounded loss, $\gamma$ in Theorem~\ref{th:improperOTB} is also unbounded. To resolve this issue, we use the clipped prediction; see \cite{cesa2001worst, foster2018logistic, sheth2020pseudo},
\begin{align*}
    \bar{p}({y|x, \theta}) = (1 - \mu)p({y|x, \theta}) + \mu p_0({y|x}) \qquad \mu \in [0, \half]~,
\end{align*}
where $p_0$ is a reference conditional density. For example, for logistic regression with two classes, we choose $p_0({y|x}) = \half$. 
We also use the corresponding smoothed logarithmic loss 
\[
    \ell_{\mu, {p_0}}(p(y|x, \theta)) = -\log\big((1 - \mu)p(y|x, \theta) + \mu {p_0}(y|x)\big) \qquad \mu \in [0, \half]~.
\]
The following lemma relates the smoothed logarithmic loss to the logarithmic loss; see also \cite{cesa2001worst} and \cite[Lemma~16]{foster2018logistic}.
\begin{lemma}\label{lem:clipping}
For any $\mu \in \big[0, \half\big]$, we have \[\log(p(y|x, \theta)) + \ell_{\mu, p_0}(p(y|x, \theta)) \leq  2 \mu.
\]
\end{lemma}
\begin{proof}
We have that 
\begin{align*}
     \ell_{\mu, {p_0}}(p(y|x, \theta)) - (-\log(p(y|x, \theta))) & = \log\left(\frac{p(y|x, \theta)}{(1 - \mu)p(y|x, \theta) + \mu p_0(y|x)}\right) \\
    & \leq \log\left(\frac{1}{1 - \mu}\right) \leq 2 \mu~,
\end{align*}
where the last inequality is due to $1-\frac{1}{y} \le \log y$ for $y \ge 0$ and that $1/(1-\mu) - 1 = \mu/(1-\mu) \leq 2 \mu$ for $\mu \in [0, \half]$.
\end{proof}
We now find the following result as a consequence of Theorem~\ref{th:improperOTB}. 
\newline
\begin{proposition}\label{cor:density}
Let
\[
\bar{p}_T(y|x) = \frac{1}{T}\sumT \E_{\theta \sim P_t}[p(y|x, \theta)]~,
\]
where $P_t$ is the distribution in round $t$ generated by EWA with initial distribution $P_1$ when run on losses $\tilde{\ell}_1, \ldots, \tilde{\ell}_{t-1}$ defined by
\[
\tilde{\ell}_t(p(Y_t| X_t, \theta)) = \ell_{\mu, {p_0}}\Big(\half p(Y_t|X_t, \theta) + \half \E_{\theta \sim P_t}\big[p(Y_t|X_t, \theta)\big]\Big)~,
\]
where $\mu \in [0, \half]$. Assume that almost surely $|\ell_{\mu, {p_0}}(\E_{\theta \sim P_t}p(Y_t|X_t, \theta)) - \ell_{\mu, {p_0}}(p(Y_t|X_t, \theta))| \le m$ for all $t = 1, \ldots, T - 1$ and all $\theta \in \Theta$.
Then, with probability at least $1-\delta$, $\bar{p}_T$ guarantees
\begin{align*}
    \E&\big[-\log\big(\bar{p}(Y|X)\big)\big] - \E_{\theta \sim Q}\big[-\log\big(p(Y|X, \theta)\big)\big] \\
    & \leq 
    \frac{2\KL(Q\|P_1) + 8 \max\{1, m\}\log(1/\delta)}{T} + 2\mu~.
\end{align*}
\end{proposition}
\begin{proof}
We start by observing that $\ell_{\mu, {p_0}}$ is $1$-exp-concave, which means that we may apply Theorem~\ref{th:improperOTB} with $\gamma = 4 \max\{1, m\}$ and conclude that, with probability at least $1 - \delta$,
\begin{align*}
    \E&\Big[\ell_{\mu, {p_0}}\big(\bar{p}_T(Y|X)\big) - \E_{\theta \sim Q}\big[-\log\big(p(Y|X, \theta)\big)\big]\Big]
\\&\leq
    \E\big[\ell_{\mu, {p_0}}\big(\bar{p}_T(Y|X)\big)\big] - \E_{\theta \sim Q}\big[\ell_{\mu, {p_0}}\big(p(Y|X, \theta)\big)\big]+ 2 \mu
\\&\leq
    \frac{2\KL(Q\|P_1) + 8 \max\{1, m\}\log(1/\delta)}{T} + 2\mu~,
\end{align*}
where we used Lemma~\ref{lem:clipping} for the first inequality and Lemma~\ref{lem:EWshiftloss} (in Appendix) for the second inequality.
\end{proof}

\subsection{Generalized linear models}\label{sec:glm}
Recall that a generalized linear model involves a probability density function \(p(\cdot|x, \theta)\) such that 
\[
p(y|x, \theta) = p(y|x^\top\theta),
\]
Following \cite{kakade2004online}, we use the following assumption on the curvature of $g_y(\cdot) = -\log\big(p(y|x^\top\theta = \cdot)\big)$,
\begin{equation}
\label{ass:GLM}
    \big|g''_y\big| \leq \kappa, \quad \textrm{for all} \quad y \in \Yset~.
\end{equation}
The reference class is the Euclidean ball in $\mathbb{R}^d$ with radius $b$, denoted in what follows by $\ball$. We use exponential weights with a Gaussian prior $\mathcal{N}(0, \sigma^2 I)$ with mean $\0$ and covariance matrix $\sigma^2 I$ and obtain the following result. 
\begin{corollary}\label{cor:GLM}
    In the setup of Proposition \ref{cor:density} suppose that $T \geq 2d$.
            Pick a generalized linear model such that $g_y = -\log\big(p(y|\cdot)\big)$ satisfies~(\ref{ass:GLM}). Choose the prior distribution $P_1 = \mathcal{N}(0, \sigma^2 I)$ with $\sigma^2 = \tfrac{b^2}{d}$, and let $P_t$ be the EWA distribution at round $t$ run on losses $\tilde{\ell}_1(P), \ldots, \tilde{\ell}_{t-1}(P)$ defined by
\[
 \tilde{\ell}_t(p(Y_t| X_t, \theta)) = \ell_{\mu, {p_0}}\Big(\half p(Y_t|X_t, \theta) + \half \E_{\theta \sim P_t}\big[p(Y_t|X_t, \theta)\big]\Big)~.
\]
Assume additionally that for all $t = 1, \ldots, T$,
\[
    \Big|\ell_{\mu, {p_0}}\big(\E_{\theta \sim P_t}p(Y_t| X_t, \theta)\big) - \ell_{\mu, {p_0}}(p(Y_t|X_t, \theta))\Big| \le m, \quad \textrm{and} \quad \|X\|_2 \leq r \quad \textrm{almost surely}.
\]
If $\mu = \frac{d}{T}$, then, with probability at least $1-\delta$, the density $\bar{p}_T(y|x) = \frac{1}{T}\sumT \E_{\theta \sim P_t}[p(y|x, \theta)]$ satisfies
        \begin{align*}
        \E&\Big[-\log\big(\bar{p}(Y|X)\big)\Big] - \min_{\theta \in \ball}\E\Big[-\log\big(p(Y|X^\top \theta)\big)\Big] \\
        & \leq \frac{ d\Big(3 + \log\Big(2 + \frac{\kappa (rb)^2}{d^2} T\Big)\Big) + \big(8 \log(T/d) + 8 m \big)\log(1/\delta)}{T}~.
    \end{align*}
        \end{corollary}
\begin{proof}
The key computations in the proof essentially follow \cite[Theorem~2.2]{kakade2004online}. Denote by ${\displaystyle \theta^\star = \argmin_{\theta\in\ball}\E\Big[-\log\big(p(Y|X^\top \theta)\big)\Big]}$. By Proposition~\ref{cor:density}, we have that for any distribution $Q$ over $\reals^d$, with probability at least $1 - \delta$,
\begin{align*}
    \E\big[-\log\big(\bar{p}(Y|X)\big)\big] - \E_{\theta \sim Q}\big[-\log\big(p(Y|X, \theta)\big)\big] 
    \leq 
    \frac{2\KL(Q\|P_1) + 8 \max\{1, m\}\log(1/\delta)}{T} + 2\mu~.
\end{align*}
Let $Q = \mathcal{N}(\theta^\star, \epsilon^2 I)$. By \cite[equation~(5)]{kakade2004online}, 
\begin{equation}
\label{eq:linearclasskl}
    \KL(Q\|P_1) = d \log(\sigma) + \frac{1}{2\sigma^2}\big(\|\theta^\star\|_2^2 + d\epsilon^2 \big) - \frac{d}{2} + d \log\left(\frac{1}{\epsilon}\right). 
\end{equation}
Now, as in the proof of Kakade and Ng \cite[Theorem~2.2]{kakade2004online}, we make a Taylor expansion of $\log\big(p({Y|\cdot})\big)$ around $X^\top \theta^\star$ and evaluate it at $X^\top\theta$. By taking expectation with respect to $\theta\sim Q$, using the fact that $\E_{\theta \sim Q}[\theta] = \theta^\star$, and the assumption~\eqref{ass:GLM} on the second derivative of $-\log(p(y|\cdot))$, we have that 
\begin{align*}
    -\E_{\theta \sim Q}\big[\log\big(p(Y|X^\top\theta)\big) \big] & \leq -\log\big(p(Y|X^\top \theta^\star)\big) + \frac{\kappa}{2}\E_{\theta \sim Q}\Big[\big(X^\top(\theta - \theta^\star)\big)^2\Big] \\
    & \leq -\log\big(p(Y|X^\top \theta^\star)\big) + \frac{\kappa r^2 \epsilon^2}{2}~,
\end{align*}
where in the last inequality we used that the covariance of $Q$ is given by $I \epsilon^2$ and the assumption that $\|X\| \leq r$. Thus, with probability at least $1 - \delta$, we have that 
\begin{align*}
    \E&\Big[\ell_{\mu}\big(\bar{p}_T(Y|X)\big)\Big] - \E\Big[-\log\big(p(Y|X^\top \theta^\star)\big)\Big] \\
    & \leq \frac{2d\log(\sigma/\epsilon) + \frac{1}{\sigma^2}\big(\|\theta^\star\|_2^2 + d\epsilon^2 \big) - d + \frac{T\kappa r^2 \epsilon^2}{2} + 8 \max\{1, m\}\log(1/\delta)}{T} + 2\mu~.
\end{align*}
Thus, by setting $\epsilon^2 = \frac{d\sigma^2}{2 d + T \kappa (r\sigma)^2}$ we have that, with probability at least $1 - \delta$,
\begin{align*}
    & \E\Big[\ell_{\mu}\big(\bar{p}_T(Y|X)\big)\Big] - \argmin_{\theta \in \ball}\E\Big[-\log\big(p(Y|X^\top \theta)\big)\Big] \\
    & \leq \frac{d\log\Big(2 + \frac{T \kappa (r\sigma)^2}{d}\Big) + \frac{1}{\sigma^2}\Big(\|\theta^\star\|_2^2 + \frac{d^2\sigma^2}{2 d +  T \kappa (r\sigma)^2} \Big) - d + \frac{\half T\kappa d(r\sigma)^2}{2 d +   T \kappa (r\sigma)^2}}{T} \\
    & \qquad + \frac{8 \max\{1, m\}\log(1/\delta)}{T} + 2\mu \\
    & \leq \frac{d\log\Big(2 + \frac{T \kappa (r\sigma)^2}{d}\Big) + \frac{1}{\sigma^2}\|\theta^\star\|_2^2 + 8 \max\{1, m\}\log(1/\delta)}{T} + 2\mu \\ 
    & \le \frac{ d\Big(1 + \log\Big(2 + \frac{\kappa (rb)^2}{d^2} T\Big)\Big) + 8 \max\{1, m\}\log(1/\delta)}{T} + 2 \mu~,
\end{align*}
where in the last equality we replaced $\sigma^2 = \frac{b^2}{d}$. Setting $\mu = \frac{d}{T}$ completes the proof.
\end{proof}

We further provide two natural applications of Corollary \ref{cor:GLM}.

\subsubsection{Logistic regression }\label{sec:logreg}
\begin{example}[Logistic regression]
Consider a setting of Corollary \ref{cor:GLM}.
Logistic regression is a generalized linear model where $p({y|x}^\top\theta) = s({x}^\top\theta)^{{y}}(1-s({x}^\top\theta))^{1-y}$, ${y} \in \{0, 1\}$, and $s(z) = \exp(z)/(1+\exp(z))$. It can be immediately shown that for logistic regression condition~\eqref{ass:GLM} is satisfied with $\kappa = \tfrac{1}{4}$. Choosing $p_0({y|x}) = \half$, we guarantee that, with probability at least $1 - \delta$,
\[
    \E\Big[-\log\big(\bar{p}(Y|X)\big)\Big] -  \min_{\theta \in \ball}\E\Big[-\log\big(p(Y|X^\top \theta)\big)\Big] \lesssim\ \frac{ d\log\Big( rb\sqrt{T}/d\Big) + \log(T/d)\log(1/\delta)}{T}~.
\]
\end{example}
Crucially, this bound avoids the exponential dependence on 
$r$ and $b$ that is deemed necessary for all proper estimators as shown in \cite{hazan2014logistic}. 
Our bound can be seen as a version of the result claimed by Foster et al. \cite[Theorem 5]{foster2018logistic}, whose proof was shown to be incorrect in \cite{mourtada2019improper}. More recently, Vijaykumar
 \cite[Corollary 18]{vijaykumar2021localization} proposed a batch algorithm for logistic regression guaranteeing, with probability at least $1-\delta$, a weaker excess risk bound of order
\begin{align*}
O\left(\frac{d\log(T)\Big(\log(rbT) + \log(1/\delta)\Big)}{T}\right)~.
\end{align*}

{Our result is particularly interesting from a computational standpoint. While the algorithm proposed in \cite{vijaykumar2021localization} is not likely to be implemented in polynomial time, the exponential weights and the corresponding sampling techniques used by our algorithm enable a polynomial running time. We discuss this in more detail in Section \ref{sec:compute}.}

\subsubsection{Gaussian {conditional} density estimation}\label{sec:gaussiandensity}
We consider a density estimation problem that naturally connects with canonical linear regression with Gaussian noise.
\begin{example}[Gaussian linear model]
\label{ex:gausslinmodel}
In conditional Gaussian density estimation we assume that the density is of the form 
\[
p(y|x^\top \theta) = \frac{1}{\sqrt{\pi}}e^{-(y - x^\top \theta)^2}, \qquad \theta \in \Theta_b~,\quad \|x\| \le r~.
\]
Define 
\[
p_0(y|x) = \frac{1}{2}\left(\frac{1}{\sqrt{\pi}}e^{-(y - rb)^2} + \frac{1}{\sqrt{\pi}}e^{-(y + rb)^2}\right).
\]
With this choice for $p_0$ and optimizing with respect to $\mu$, the estimator of Corollary \ref{cor:GLM} gives, with probability at least $1-\delta$,
\begin{align*}
   \E\Big[-\log\big(\bar{p}(Y|X)\big)\Big] -  \min_{\theta \in \ball}\E\Big[-\log\big(p(Y|X^\top \theta)\big)\Big] \lesssim\ \frac{ d\log\Big( rb\sqrt{T}/d \Big) + \log(T/d)\log(1/\delta)}{T}~.
\end{align*}
\end{example}
We present the corresponding calculation in Appendix \ref{compexample}.
The best known excess risk bound for this question is provided in \cite[Proposition 10]{mourtada2019improper}, where the bound scales as $O\Big(\frac{d\log(rb/\sqrt{d})}{T}\Big)$; see also the related bounds in  \cite{forster2002relative, kakade2004online}. Although the bound in \cite{mourtada2019improper} has a better dependence on $T$, it holds only in expectation as opposed to ours, which holds with high probability. We note that these authors asked about possible high-probability upper bounds in this setting.

\subsection{Estimation of discrete distributions}
\label{sec:discretedistrib}
In this section, we consider the following basic problem. Given some unknown distribution $p^{\star} \in \Delta^d$, where $\Delta^d$ denotes the set of all distribution over the finite set $[d] = \{1, \ldots, d\}$, we have $T$ independent observations each sampled according to $p^{\star}$. Our goal is to construct the distribution $\bar p$ such that $\KL(p, \bar{p})$ is as small as possible with high probability.

We work with the logarithmic loss $\ell(p, y) = \sum_{i = 1}^d -\log\big(p(i)\big)\id[y = i]$ for $y \in [d]$ and $p \in \Delta^d$. For simplicity, we assume that $T/2$ is an integer. We use the following predictor: for all $i = 1, \ldots, d$,
\begin{align}\label{eq:seqprasspred}
    \bar{p}(i) = (1 - \mu)\bar{p}_T(i) + \frac{\mu}{d} \qquad \textnormal{with} \qquad
    \bar{p}_T(i) = \frac{1}{T/2}\sum_{t=T/2+1}^{T} \E_{p \sim P_{t}}[p(i)]~,
\end{align}
where, for $t > T/2$, 
\[
P_t(p) = \frac{P_1(p)\exp\Big(-\sum_{s = T/2 + 1}^t\tilde{\ell}_s(p, Y_s)\Big)dp}{\int\limits_{\Delta^d}P_1(p)\exp\Big(-\sum_{s = T/2 + 1}^t\tilde{\ell}_s(p, Y_s)\Big)dp}~.
\]
is the exponential weights distribution on shifted losses 
\[
\tilde\ell_t(p, y) = \sum_{i = 1}^d -\log\Big(\half(1-\mu)p(i) + \half(1 - \mu)\E_{p \sim P_t}p(i) + \tfrac{\mu}{d}\Big)\id[y = i]~,
\]
and where we use the data dependent prior
\begin{align*}
    P_1(p) = \frac{P_0(p)\exp\Big(\half \sum_{i=1}^d n_{T/2}(i)\log(p(i))\Big)}{\E_{p \sim P_0}\Big[\exp\Big(\half \sum_{j=1}^d n_{T/2}(j)\log(p(j))\Big) \Big]}~,
\end{align*}
where $n_{T/2}(i) = \sum_{t = 1}^{T/2} \id[Y_t = i]$ and $P_0$ is a Dirichlet density with parameters $z_1 = \ldots = z_d = \half$; see the formal details in what follows.

Some remarks are in order. The approach in this section is based on suffix averaging \cite{rakhlin2012making,harvey2019tight, aden2023optimal}: we only run the exponential weights algorithm on shifted losses from rounds $T/2$ onward with a prior constructed using the first $T/2$ observations. This does not affect the application of Proposition~\ref{cor:density}. However, it does affect the way in which the $2\KL(Q\|P_1)$ term, which we obtain from Proposition~\ref{cor:density}, is treated in the proof of the next theorem.

\begin{theorem}
\label{prop:kldiv}
    Suppose that $T > 4d$, and let $p^{\star} \in \Delta^d$ denote the unknown distribution of the observations. Then, Predictor~\eqref{eq:seqprasspred} with $\mu = \frac{d}{T}$ guarantees that, with probability at least $1 - 2\delta$,
        \begin{align*}
        \KL(p^{\star}\|\bar{p}) \leq \frac{22d +  28\log(T)\log(1/\delta)}{T}~.
    \end{align*}
    \end{theorem}
\begin{proof}
    Observe that $P_1$ depends only on the first $T/2$ observations. Thus, conditioned on the realization of these $T/2$ observations, applying Proposition~\ref{cor:density}, we have that with probability at least $1 - \delta$,
       \[
        \E\big[\ell(\bar{p}, Y)\big] \leq  \E_{p\sim Q}\Big[\E\big[\ell(p, Y)\big]\Big]+\frac{2\KL(Q\|P_1) + 8 \max\{1, \log(d/\mu)\}\log(1/\delta)}{T/2} + 2\mu~.
    \]
    We want to choose the optimal $Q$ and bound the right-hand side of this inequality. Observe that $\E_{p\sim Q}\left[\E\big[\ell(p, Y)\big]\right] = \E_{p\sim Q}\left[\sum_{i=1}^d -p^{\star}(i)\log(p(i))\right]$.
                                                        By the Donsker-Varadhan variational inequality the optimal choice of the distribution $Q$ satisfies
    \begin{align*}
        \frac{T}{4}\E_{p\sim Q}\Bigg[\sum_{i=1}^d -p^{\star}(i)\log(p(i))\Bigg] + \KL(Q\|P_1)
        = -\ln\E_{p \sim P_1}\Big[\exp\Big({ \sum_{i=1}^d p^{\star}(i)\frac{T}{4}\log(p(i))}\Big)\Big]~,
    \end{align*}
        Recalling the definition of $P_1$
                            and because $P_0$ is a Dirichlet density with parameters $z_1 = \ldots = z_d = \half$, we have that
        \begin{align*}
        & -\ln\E_{p \sim P_1}\Big[\exp\Big( \sum_{i=1}^d p^{\star}(i)\frac{T}{4}\log(p(i))\Big)\Big] \\
        & = -\ln \left(\frac{\E_{p \sim P_0}\Big[\exp\Big(\half \sum_{i=1}^d \big(p^{\star}(i) \tfrac{T}{2} + n_{T/2}(i)\big)\log(p(i))\Big) \Big]}{\E_{p \sim P_0}\Big[\exp\Big(\half \sum_{i=1}^d n_{T/2}(i)\log(p(i))\Big) \Big]}\right) \\
        & = -\ln \left( \frac{\Gamma\big(\tfrac{T}{4} + \tfrac{d}{2}\big)\prod_{i = 1}^d \Gamma\big(\half + p^{\star}(i) \tfrac{T}{4} + \tfrac{1}{2} n_{T/2}(i)\big)}{\Gamma\big(\tfrac{T + d}{2}\big)\prod_{i = 1}^d \Gamma\big(\half + \half n_{T/2}(i)\big)}\right)~,
    \end{align*}
        where we used the general formula for the moments of the Dirichlet distribution. Recall that by Stirling's approximation we can write for all $x \ge 1/2$,
    \[
\sqrt{2\pi}x^{x - 1/2}\exp(- x) \le \Gamma(x) \le \sqrt{2\pi}x^{x - 1/2}\exp(- x + 1/(12x)) \le \sqrt{2\pi}x^{x - 1/2}\exp(- x + 1/6).
    \]
Applying this bound, and using $x \log x \le x\log(x + 1/2) \le x\log x + 1/2$ for all $x > 0$, we have
        \begin{align*}
        &-\frac{1}{T}\ln \left( \frac{\Gamma(\tfrac{T}{4} + \tfrac{d}{2})\prod_{i = 1}^d \Gamma\big(\half + p^{\star}(i) \tfrac{T}{4} + \tfrac{1}{2} n_{T/2}(i)\big)}{\Gamma(\tfrac{T + d}{2})\prod_{i = 1}^d \Gamma\big(\half + \half n_{T/2}(i)\big)}\right) \\
        & = -\frac{1}{T}\ln \left( \frac{\Gamma(\tfrac{T}{4} + \tfrac{d}{2})}{\Gamma(\tfrac{T + d}{2})}\right) - \frac{1}{T}\log\left(\frac{\prod_{i = 1}^d \Gamma\big(\half + p^{\star}(i) \tfrac{T}{4} + \tfrac{1}{2} n_{T/2}(i)\big)}{\prod_{i = 1}^d \Gamma\big(\half + \half n_{T/2}(i)\big)}\right)\\
        &\leq \frac{1}{T}\left(-\left(\frac{T}{4} + \frac{d}{2} - \frac{1}{2}\right)\log\left(\frac{T}{4} + \frac{d}{2}\right) + \frac{1}{6} + \left(\frac{T}{2} + \frac{d}{2} - \frac{1}{2}\right)\log\left(\frac{T}{2} + \frac{d}{2}\right) - \frac{T}{4}\right) \\
        & \qquad + \frac{1}{T}\bigg(\frac{d}{6} + \frac{T}{4} - \sum_{i = 1}^d \half \big(p^{\star}(i) \tfrac{T}{2} +  n_{T/2}(i)\big)\log\Big({\half \big(1 + p^{\star}(i) \tfrac{T}{2} +  n_{T/2}(i) \big)}\Big)
        \\
        &\quad\qquad+ \sum_{i = 1}^d \half n_{T/2}(i) \log\Big({\half\left( 1 +  n_{T/2}(i)\right)}\Big)\bigg) \\
	    &\leq \frac{1}{T}\left(-\left(\frac{T}{4} + \frac{d}{2} - \frac{1}{2}\right)\log\left(\frac{T}{4} + \frac{d}{2}\right) + \frac{1}{6} + \left(\frac{T}{2} + \frac{d}{2} - \frac{1}{2}\right)\log\left(\frac{T}{2} + \frac{d}{2}\right)\right)  \\
        & \qquad + \frac{1}{T}\bigg(\frac{2d}{3} - \sum_{i = 1}^d \half \big(p^{\star}(i) \tfrac{T}{2} +  n_{T/2}(i)\big)\log\Big({\half \big(p^{\star}(i) \tfrac{T}{2} +  n_{T/2}(i) \big)}\Big) 
        \\
        &\qquad\quad + \sum_{i = 1}^d \half n_{T/2}(i) \log\Big({\half n_{T/2}(i)}\Big) \bigg) \\
        &= \frac{1}{T}\left(-\left(\frac{T}{4} + \frac{d}{2} - \frac{1}{2}\right)\log\left(\frac{T}{4} + \frac{d}{2}\right) + \frac{1}{6} + \frac{2d}{3} + \left(\frac{T}{2} + \frac{d}{2} - \frac{1}{2}\right)\log\left(\frac{T}{2} + \frac{d}{2}\right) \right) \\
        & \qquad + \frac{1}{T}\bigg(\frac{T}{2}H(\half p^{\star} + \half \hat{p}) - \frac{T}{2}\log\left(\frac{T}{2}\right) - \frac{T}{4} H(\hat{p}) + \frac{T}{4}\log\left(\frac{T}{4}\right) \bigg)~,
    \end{align*}
        where for any $p \in \Delta^d$, $H(p) = -\sum_{i = 1}^d p(i) \log(p(i))$ denotes the entropy and $\hat{p}(i) = \frac{n_{T/2}(i)}{T/2}$. 
    Combining four terms that involve logarithms, we obtain
    \begin{align*}
&-\left(\frac{T}{4} + \frac{d}{2} - \frac{1}{2}\right)\log\left(\frac{T}{4} + \frac{d}{2}\right) + \left(\frac{T}{2} + \frac{d}{2} - \frac{1}{2}\right)\log\left(\frac{T}{2} + \frac{d}{2}\right) - \frac{T}{2}\log\left(\frac{T}{2}\right) + \frac{T}{4}\log\left(\frac{T}{4}\right) 
\\
&\le \frac{T}{2}\log\left(1 + \frac{d}{T}\right) + \frac{d}{2}\log\left(2\right) \le \frac{2d}{3}~.
    \end{align*}
    Thus, summarizing what we have obtained so far, we have, with probability at least $1 - \delta$,
        \[\E\big[\ell(\bar{p}, Y)\big] \leq \frac{16 d + 2 + 48 \max\{1, \log(d/\mu)\}\log(1/\delta)}{3T} + 2\mu  +  2H(\half \hat{p} + \half p^{\star}) - H(\hat{p})~.
    \]
        Now, using the concavity of $H(\cdot)$ together with the formula for the Bregman divergence of the negative entropy $-H(\cdot)$, we have
    \begin{align*}
 2H(\half \hat{p} + \half p^{\star}) - H(\hat{p}) &= 2H(p^{\star}) - H(\hat{p}) + 2H(\half \hat{p} + \half p^{\star}) - 2H(p^{\star}) 
 \\
 &\le 2H(p^{\star}) - H(\hat{p}) + \nabla H(p^{\star})^{\top}(\hat{p} - p^{\star})
 \\
 &= H(p^{\star}) 
+ \KL(\hat{p}\|p^{\star})~.
    \end{align*}
        It is only left to provide a high probability bound on $\KL(\hat{p}\|p^{\star})$.
    Using \cite[Corollary 1.7]{agrawal2022finite} we have that, with probability at least $1-\delta$
        \begin{align*}
         \KL(\hat{p}\|p^{\star}) & \leq \E[ \KL(\hat{p}\|p^{\star})] + \frac{6d + 6\log(1/\delta)}{T/2} \\
         & \leq \frac{14d  + 12 \log(1/\delta)}{T}~,
    \end{align*}
        where in the second inequality we used $\E[ \KL(\hat{p}\|q)] \leq \frac{d-1}{T/2}$ (see \cite[Section~4]{paninski2003estimation}). By the union bound, we can therefore conclude that, with probability at least $1 - 2\delta$,
    \[\E\big[\ell(\bar{p}, Y)\big] \leq H(p^{\star}) + \frac{60 d + 84 \max\{1, \log(d/\mu)\}\log(1/\delta)}{3T} + 2\mu~.
\]
        which completes the proof after we choose $\mu = d/T$.
\end{proof}

We now put our result in the context and compare with several previous bounds. The question studied in this section was historically first explored in a sequential setup, given its connections to universal coding. In this setting, we work with logarithmic loss and aim to minimize regret over any sequence of length $T$. For $d = 2$, the celebrated estimator of Krichevsky and Trofimov \cite{krichevsky1981performance}, extended later for all $d \ge 2$ by Xie and Barron \cite{xie1997minimax}, provides a sharp regret bound that scales as $\frac{d-1}{2}\log(T)$ plus some lower-order terms. For a more comprehensive exploration of the topic, we refer to \cite{shtar1987universal, merhav98, rissanen1996fisher} and the monographs \cite{cesa2006prediction,grunwald2007minimum, polyanskiy2023}.  Evidently, our result does not directly arise from these existing sequential bounds due to the presence of a multiplicative logarithmic factor, $\log T$. We additionally remark that suffix averaging can be seen as a general way to address the question of Gr\"{u}nwald and Kot{\l}owski \cite{pmlr-v19-grunwald11b}, which involves proving sharp (without additional logarithmic factors) excess risk bounds for statistical problems with logarithmic loss.

The statistical problem we are delving into is more complex. Braess and Sauer \cite{braess2004bernstein} provided a bound on the expected value of the Kullback-Leibler divergence in our setting with the optimal leading term $\frac{d - 1}{2T}$. Their estimator was described in \cite{kamath2015learning} as \say{somewhat impenetrable, with its proof relying on automated computer calculations}. A simpler \emph{Laplace} estimator achieves a slightly weaker in-expectation upper bound $\frac{d - 1}{T}$ as shown in \cite{catoni1997mixture, mourtada2019improper}. See also a similar bound in \cite{forster2002relative} in the case where $d = 2$.

High probability guarantees are currently only known for this same Laplace estimator, and are provided in \cite{bhattacharyya2021near,canonne2023concentration}. The latter result applies\footnote{Of note, the authors of \cite{canonne2023concentration} focus on sharp concentration inequalities, so their high probability bound actually has the exact leading term $\frac{d - 1}{T}$, whereas our bound has a larger constant in front of this term. Simultaneously, considering the optimal bound in \cite{braess2004bernstein}, there is a substantial interest in obtaining high-probability bounds with the optimal leading term $\frac{d - 1}{2T}$.} to the same Laplace estimator, denoted as $\bar{p}_{\operatorname{L}}$, thus providing the previously best known high probability upper bound within our context, as follows:
\[
\KL(p^{\star}\|\bar{p}_{\operatorname{L}}) \lesssim \frac{d + \sqrt{d\log^5(1/\delta)}}{T}~.
\]

Our result supplements this bounds and gives improvements in many regimes. We further note that our analysis aligns more with classical results in \cite{krichevsky1981performance,xie1997minimax}, interpreted as an exponential weights algorithm with Dirichlet priors. The key distinction in our case is the second-order correction we employ in Theorem \ref{th:improperOTB}, the truncation of the logarithmic loss to make it bounded, along with the suffix averaging technique to eliminate the unnecessary multiplicative $\log T$ term stemming from sequential prediction analysis.

\section{Model aggregation with bounded exp-concave losses}\label{sec:modelselagg}

In this section, we discuss an application of our results to the setup of model aggregation. This setup was formally introduced by Nemirovski \cite{nemirovski2000topics} and further studied by Tsybakov \cite{tsybakov2003optimal} and several other works that we discuss in what follows. Some early papers on this question, where the online to batch approach was a part of the analysis, include \cite{yang1999information, catoni1997mixture, yang2000mixing}, \cite[Chapter 3]{catoni2004statistical}. Assume that we are given a finite dictionary $\mathcal F = \{f_1, \ldots, f_K\}$ of real-valued absolutely bounded functions defined on the instance space $\mathcal X$. In model selection (MS) aggregation, one is interested in constructing an estimator $\bar{f}_T$ based on the i.i.d.\ sample $(X_t, Y_t)_{t = 1}^T$ such that, with probability at least $1 - \delta$,
\begin{equation}
\label{eq:optrateofaggregation}
R(\bar{f}_T) - \min\limits_{f \in \mathcal F}R(f) = O\left(\frac{\log(K) + \log(1/\delta)}{T}\right)
\end{equation}
under appropriate boundedness and curvature assumptions on the loss function $\ell$. Following Tsybakov \cite{tsybakov2003optimal}, the bound of the form~\eqref{eq:optrateofaggregation} will be called the \emph{optimal rate of aggregation}. Our next result provides a simple estimator that achieves the optimal rate of aggregation for general bounded exp-concave loss. 
\begin{proposition}\label{cor:modelselection}
Suppose that the loss $\ell: \domainw \times \Yset \mapsto \mathbb{R}$ satisfies the assumptions of Theorem \ref{th:improperOTB}. Let $\bar{f}_T = \frac{1}{T}\sumT \E_{f \sim P_t}[f]$, where $P_t$ is the Exponential Weights distribution at round $t$ on losses $\tilde{\ell}_1(f), \ldots, \tilde{\ell}_{t-1}(f)$, where $\tilde{\ell}_t(f) = \ell\big(\half f(X_t) + \half \E_{f \sim P_t}[f(X_t)], Y_t\big)$ and $P_1$ is a uniform prior distribution over a finite set $\Fset$ of size $K$. With probability at least $ 1- \delta$, $\bar{f}_T$ guarantees
\begin{align*}
    R(\bar{f}_T(X), Y) - \min\limits_{f \in \mathcal F}R(f) \leq \frac{\frac{2}{\alpha} \ln(K) + 8\max\{\frac{1}{\alpha}, m\}\log(1/\delta)}{T}~.
\end{align*}
\end{proposition}
\begin{proof}
The proof of Proposition~\ref{cor:modelselection} follows immediately from Theorem~\ref{th:improperOTB} and Lemma~\ref{lem:EWshiftloss}. 
\end{proof}
The bound of Proposition \ref{cor:modelselection} is nontrivial to obtain in general. As we mentioned, any proper estimator, which takes values in $\mathcal F$, fails due to the lower bound $\Omega\big(\frac{1}{\sqrt{T}}\big)$. However, for the squared loss or strongly convex losses, several algorithms have been developed and analyzed over the years that achieve the optimal rate of aggregation \eqref{eq:optrateofaggregation}, as evidenced by the bounds in \cite{audibert2007progressive, lecue2009aggregation, lecue2014optimal, liang2015learning, wintenberger2017optimal, van2022regret, kanade2022exponential}. When applied to the special case of bounded squared loss, our analysis is arguably the simplest among the existing estimators that achieve the optimal bound \eqref{eq:optrateofaggregation}.

While Gaillard and Wintenberger \cite{gaillard2018efficient} present a result in a setup that is similar to ours for general exp-concave losses, their bound includes an additional $O(\log\log T)$ factor and depends on the assumption that the gradient of the loss is bounded.

\section{Linear regression}\label{sec:linleg}

In this section, we consider linear regression with the squared loss $\ell(\theta^\top X, Y) = (\theta^\top X - Y)^2$. We assume that $(X, Y)$ is such that $X$ is a random vector in $\mathbb{R}^d$ with $\|X\| \le r$ almost surely for some $r > 0$ and $Y$ is a random variable satisfying $|Y|\le \ymax$ almost surely. In what follows, we make no assumptions on the dependence between $X$ and $Y$. Our reference class is parameterized by $\Theta_b$ defined by
\[
\Theta_b = \{\theta \in \mathbb{R}^d: \|\theta\| \le b\}~.
\]
We first discuss the most natural estimator, which is linear least squares constrained to the set $\Theta_b$. Denote 
\[
\widehat{\theta}_{\operatorname{ERM}} = \argmin_{\theta \in \Theta_b}\frac{1}{T}\sum\limits_{t = 1}^T(Y_t - \theta^\top X_t)^2~.
\]
The standard local Rademacher complexity bound---see \cite{bartlett2005local} and \cite{shamir2015sample,vavskevivcius2020suboptimality} for exact statements---implies that, with probability at least $1-\delta$,
\[
\E\Big[\big(\widehat{\theta}_{\operatorname{ERM}}^\top X - Y\big)^2\Big] - \inf\limits_{\theta \in \Theta_b}\E\big[(\theta^\top X - Y)^2\big] \lesssim (\ymax + rb)^2\frac{d + \log(1/\delta)}{T}~,
\]
{where the expectation is taken with respect to $(X, Y)$.} Interestingly, when using improper learners, the dependence on some of the parameters can be significantly improved. In fact, Va\vs kevi\v{c}ius
 and Zhivotovskiy \cite{vavskevivcius2020suboptimality} noticed that, once properly tuned, the Vovk-Azoury-Warmuth (see \cite{vovk2001competitive, azoury2001relative}) estimator achieves an \emph{in-expectation} excess risk bound of the form
\begin{equation}
\label{eq:inexpexcessrisk}
O\left(\frac{d\ymax^2}{T}\log\left(\frac{rb\sqrt{T}}{d\ymax}\right)\right)~.
\end{equation}
This already provides an exponential improvement in the dependence on $r$ and $b$.
However, the standard online to batch conversion used to prove this bound does not lead to a high-probability bound. The work of Mourtada, Va\vs kevi\v{c}ius
 and Zhivotovskiy \cite{mourtada2021distribution} showed that at least for some distributions the standard online to batch conversion of the Vovk-Azoury-Warmuth algorithm leads to constant excess risk with constant probability. Furthermore, the Vovk-Azoury-Warmuth algorithm produces improper predictions, which means that standard confidence boosting approaches, like the one suggested in \cite{mehta2017fast}, cannot be applied.

Our next result shows for the first time that we can get the same guarantee as in equation~\eqref{eq:inexpexcessrisk} with high probability. 
Our predictions make use of clipping, which is defined as
\begin{align*}
    \clip_{\ymax}(z) = \begin{cases}
        -\ymax & \textnormal{if } ~ z \in (-\infty, -\ymax), \\
        z & \textnormal{if } ~ z \in [-\ymax, \ymax], \\
        \ymax & \textnormal{if } ~ z \in (\ymax, \infty).
    \end{cases}
\end{align*}
Our modification to the predictions is the same as used by Forster in \cite{forster1999relative}, who also uses clipped predictions. Let $y_t(\theta) = \clip_{\ymax}(\theta^\top X_t)$. For any given $x$, our algorithm predicts with
\begin{align}\label{eq:squaredlosspredictions}
    \ybar_T(x) = \frac{1}{T}\sumT \E_{\theta \sim P_t}\big[\clip_{\ymax}(\theta^\top x)\big]~,
\end{align}
where
\begin{align}\label{eq:ewquad}
    dP_{t+1}(\theta) = \frac{e^{-\frac{1}{8\ymax^2} \sumt \big(\half y_s(\theta) + \half \E_{\theta \sim P_s}[y_s(\theta)] - Y_s\big)^2}dP_1(\theta)}{\int e^{-\frac{1}{8\ymax^2} \sumt \big(\half y_s(\theta) + \half \E_{\theta \sim P_s}[y_s(\theta)] - Y_s\big)^2}dP_1(\theta)}~,
\end{align}
and $P_1$
is the Gaussian distribution with mean $\0$ and covariance matrix $\sigma^2 I$ for some $\sigma > 0$. 
\begin{proposition}
\label{prop:whpvaw}
    Suppose that $\|X\|_2 \leq r$ and that $|Y| \leq \ymax$ almost surely. With probability at least $1 - \delta$, predictor~\eqref{eq:squaredlosspredictions} with $\sigma^2 = \frac{b^2}{d}$ satisfies 
        \begin{align*}
        \E&\big[(\ybar_T(X) - Y)^2\big] - \inf_{\theta \in \ball} \E\big[(X^\top\theta - Y)^2\big]
    \\ &\leq
        \frac{8\ymax^2 d}{T}\left(1 + \log\left(2 + \left(\frac{rb}{2ld}\right)^2T\right)\right)
        +\frac{64\ymax^2 \log(1/\delta)}{T}~.
    \end{align*}
    \end{proposition}
\begin{proof}
    Denote $\theta^\star = \argmin_{\theta \in \ball} \E\big[(\theta^\top X-Y)^2\big]$.
    We first prove that
        \begin{align}\label{eq:clipbetter}
        \big(\clip_{\ymax}(z) - y\big)^2 - (z - y)^2 \leq 0~,
    \end{align}
        for any $y \in [-\ymax, \ymax]$. If $z \in [-\ymax, \ymax]$ then $\big(\clip_{\ymax}(z) - y\big)^2 - (z - y)^2 = 0$ and so we only need to worry about $z \not \in [-\ymax, \ymax]$. We will prove the inequality for $z > \ymax$, the case where $z < -\ymax$ follows from symmetric arguments. Since $z > \clip_{\ymax}(z) = \ymax \geq y$ we have that $\ymax + z - 2y > 0$ and $\ymax - z < 0$. Therefore, 
        \begin{align*}
        \big(\clip_{\ymax}(z) - y\big)^2 - (z - y)^2 & = \big(\clip_{\ymax}(z) + z - 2y\big)\big(\clip_{\ymax}(z) - z\big) \\
        & = (\ymax + z - 2y)(\ymax - z) \leq 0~.
    \end{align*}
        Let $Q = \mathcal{N}(\theta^\star, \epsilon^2 I)$. We have that 
        \begin{align*}        \E\Big[\E_{\theta \sim Q}\big[(X^\top\theta - Y)^2\big]\Big] -\E\big[(X^\top\theta^\star - Y)^2\big] & = \E\Big[\E_{\theta \sim Q}\big[(X^\top(\theta + \theta^\star) - 2Y)^{\top}X^\top(\theta - \theta^\star)\big]\Big] \nonumber \\
        & \le \epsilon^2 \E[X^\top X] \leq \epsilon^2 r^2~.
    \end{align*}
        This means that
        \begin{align*}
        \E&\big[(\ybar_T(X) - Y)^2\big] - \E\big[(X^\top\theta^\star - Y)^2\big] \\
        & \leq \E\big[(\ybar_T(X) - Y)^2\big] - \E\Big[\E_{\theta \sim Q}\big[(X^\top\theta - Y)^2\big]\Big] + \epsilon^2 r^2 \\
        & \leq \E\big[(\ybar_T(X) - Y)^2\big] - \E\Big[\E_{\theta \sim Q}\big[(\clip_{\ymax}(X^\top\theta) - Y)^2\big]\Big] + \epsilon^2 r^2,
    \end{align*}
        where the second inequality is due to~\eqref{eq:clipbetter}. With our predictions the squared loss is $(8\ymax^2)^{-1}$ exp-concave since the second derivative of $h(z) = (z-y)^2$ is 2 and the first derivative is $2(z-y)$, which means that with $\alpha = \tfrac{1}{8\ymax^2}$ equation~\eqref{eq:defexpconcave} is satisfied. We now apply Theorem~\ref{th:improperOTB} with $\gamma = 32 \ymax^2$ to find that, with probability at least $1 - \delta$,
        \begin{align*}
        \E\big[(\ybar_T(X) - Y)^2\big] - \E\Big[\E_{\theta \sim Q}\big[\big(\clip_{\ymax}(X^\top\theta) - Y \big)^2\big]\Big] \leq \frac{2R_T + 64 \ymax^2 \log(1/\delta)}{T}~.
    \end{align*}
        Distribution $P_{t+1}$ in equation~\eqref{eq:ewquad} is the exponential weights distribution on the shifted squared losses $\sumt \tilde{\ell}_s(y_s(\theta)) = \sumt \big(\half y_s(\theta) + \half \E_{\theta \sim P_s}[y_s(\theta)] - Y_s\big)^2$. Therefore, by \eqref{eq:linearclasskl}, we have that 
        \begin{align*}
        R_T \leq 8 \ymax^2 \KL(Q\|P_1) = 8 \ymax^2\left(d \log(\sigma) + \frac{1}{2\sigma^2}\Big(\|\theta^\star\|_2^2 + d\epsilon^2 \Big) - \frac{d}{2} + d \log\left(\frac{1}{\epsilon}\right)\right)~.
    \end{align*}
    Combining the above we find that with probability at least $1 - \delta$,
        \begin{align*}
        \E&\big[(\ybar_T(X) - Y)^2\big] - \E\big[(X^\top\theta^\star - Y)^2\big] \\
        & \leq \frac{8\ymax^2\Big(d \log\Big(\frac{\sigma^2}{\epsilon^2}\Big) + \frac{1}{\sigma^2}\big(\|\theta^\star\|_2^2 + d\epsilon^2 \big) - d\Big) + T\epsilon^2 r^2 + 64 \ymax^2 \log(1/\delta)}{T}~.
    \end{align*}
        Next, set $\epsilon^2 = \frac{d\sigma^2}{2d + (T r^2 \sigma^2)/(4\ymax^2)}$ to find that
        \begin{align*}
        8\ymax^2&\left(d \log\left(\frac{\sigma^2}{\epsilon^2}\right) + \frac{1}{\sigma^2}\big(\|\theta^\star\|_2^2 + d\epsilon^2 \big) - d\right) + T\epsilon^2 r^2 \\
        & \leq 8\ymax^2 d \log\left(2 + \frac{T r^2 \sigma^2}{4l^2d}\right) + \frac{8\ymax^2}{\sigma^2}\|\theta^\star\|_2^2 \\
        & \leq 8\ymax^2 d\left(1 + \log\Big(2 + \frac{T r^2 b^2}{4 l^2d^2}\Big)\right)~,
    \end{align*}
        where in the last inequality we used $\|\theta^\star\|_2^2 \leq b^2$ and $\sigma^2 = \frac{b^2}{d}$.
\end{proof}

To put our result in context, we should recall a recent result of Mourtada, Va\vs kevi\v{c}ius
 and Zhivotovskiy presented in \cite{mourtada2021distribution}. Their results imply that there is an improper estimator, whose output will be denoted by $\ybar_{\textrm{MVZ}}$, such that, with probability at least $1 - \delta$,
 \begin{equation}
\label{eq:excessriskofmvzh}
\E\big[(\ybar_{\textrm{MVZ}}(X) - Y)^2\big] - \inf_{\theta \in \mathbb{R}^d} \E\big[(X^\top\theta - Y)^2\big] \lesssim
\frac{\ymax^2(d\log\left(T/d\right) + \log(1/\delta))}{T}~.
\end{equation}
Observe that this bound depends neither on the distribution of $X$ nor on the norm of the target parameter. 
Although our result gives a slightly weaker statistical bound, it might have some computational advantage over the estimator in the bound \eqref{eq:excessriskofmvzh}. We discuss this briefly in Section~\ref{sec:compute}.

In light of the increasing interest in computationally efficient algorithms that can provide high-probability excess risk bounds, we revisit the Vovk-Azoury-Warmuth algorithm. Previously discussed, it currently lacks such a high-probability bound. We recursively define the following version of this algorithm. We let
\[
\theta_{\vaw, t}(x) = \left(\frac{1}{4} xx^\top + \sum_{s=1}^{t-1} \frac{1}{4} X_s X_s^\top + \frac{1}{\sigma^2} I\right)^{-1}\sum_{s = 1}^{t-1} \frac{1}{2} \tilde{Y}_{s}X_{s}
\]
denote the parameter of the Vovk-Azoury-Warmuth algorithm,
where \[
\tilde{Y}_t = -\half \clip_{\ymax}\big({\theta}_{\vaw, t}(X_t)^\top X_t\big) + Y_t~.
\]
The value $\tilde{Y}_t$ at time $t$ is based on the predictions made by our estimator from time $1$ to $t - 1$ and the value $Y_t$. Our final prediction is the re-weighted average across the trajectory. It can be expressed as follows:
\begin{align}\label{eq:VAWclipped2}
    \ybar_T(x) = \frac{1}{T}\sumT \ \clip_{\ymax}\big(\theta_{\vaw, t}(x)^\top x\big)~.
\end{align}
This forecaster can be computed in $O(d^2 T)$ time: by using the Sherman-Morrison formula one can update from $\theta_{\vaw, t}(x)$ to $\theta_{\vaw, t+1}(x)$ in $O(d^2)$ time, see, for example, Algorithm~2 in \cite{VanErven2021metagrad} and the discussion surrounding that algorithm. Using the forecaster in equation~\eqref{eq:VAWclipped2} leads to the result in Proposition \ref{prop:efficientvaw2}. Proposition \ref{prop:efficientvaw2} provides a computationally efficient estimator, but its excess risk bound is weaker in terms of the dependence on $r$ compared to Proposition \ref{prop:whpvaw}.
\begin{proposition}
\label{prop:efficientvaw2}
Denote $\theta^\star = \argmin_{\theta \in \ball} \E\big[(\theta^\top X-Y)^2\big]$.
    In the setup of Proposition \ref{prop:whpvaw} the following holds. With probability at least $1 - \delta$, predictor~\eqref{eq:VAWclipped2} with $\sigma^2 = \frac{b^2}{d\ymax^2}$ satisfies 
        \begin{align*}
        \E\big[(\ybar_T(X) - Y)^2\big] - \E\Big[\big(X^\top\theta^\star - Y \big)^2\Big] \leq \frac{8\ymax^2 d\log(1 + T \frac{b^2r^2}{4d^2\ymax^2} )  + 64 \max\{\ymax^2, b^2 r^2\} \log(1/\delta)}{T}~.
    \end{align*}
    \end{proposition}
\begin{proof}
    We first prove that 
        \begin{align}\label{eq:clipbetter3}
        \big(\clip_{\ymax}(z) - y\big)^2 - \big(\half \clip_{\ymax}(z) + \half z - y\big)^2 \leq 0~,
    \end{align}
        for any $y \in [-\ymax, \ymax]$. If $z \in [-\ymax, \ymax]$ then $\clip_{\ymax}(z) = z$ and so we only need to worry about $z \not \in [-\ymax, \ymax]$. We will prove the inequality for $z > \ymax$, the case where $z < -\ymax$ follows from symmetric arguments. Since $z > \clip_{\ymax}(z) = \ymax \geq y$ we have that $\tfrac{3}{2} \ymax + \half z - 2y > 0$ and $\ymax - z < 0$. Therefore, 
        \begin{align*}
        \big(\clip_{\ymax}(z) - y\big)^2 - \big(\half \clip_{\ymax}(z) + \half z - y\big)^2 & = \Big(\tfrac{3}{2} \clip_{\ymax}(z) + \half z - 2y\Big)\Big(\half \clip_{\ymax}(z) - \half z\Big) \\
        & = \half \Big(\tfrac{3}{2} \ymax + \half z - 2y\Big)(\ymax - z) \leq 0~.
    \end{align*}
        Since with the clipped predictor the squared loss is $8 \max\{b^2r^2, \ymax^2\}$-exp concave (the second derivative of $f(z) = (z-y)^2$ is 2 and the first derivative is $2(z-y)$), we may now apply Theorem~\ref{th:improperOTB} with $\gamma = 32 \max\{\ymax^2, b^2 r^2\}$, and $Q$ being a point-mass on $\theta^\star$ to find that, with probability at least $1 - \delta$,
        \begin{align}\label{eq:VAWhp1}
        \E\big[(\ybar_T(X) - Y)^2\big] - \E\Big[\big(X^\top\theta^\star - Y \big)^2\Big] \leq \frac{2R_T + 64 \max\{\ymax^2, b^2 r^2\} \log(1/\delta)}{T}~,
    \end{align}
        where, using the definition \eqref{eq:regretreq}, the shifted regret is given by
    \[
R_T = \sumT \bigg(\Big(\clip_{\ymax}\big(\theta_{\vaw, t}(X_t)^\top X_t\big) - Y_t\Big)^2 - \Big(\half \clip_{\ymax}\big(\theta_{\vaw, t}(X_t)^\top X_t\big) + \half X_t^\top\theta^\star - Y_t \Big)^2\bigg)~.
    \]
    It is only left to bound $R_T$. We apply equation~\eqref{eq:clipbetter3} to find 
        \begin{align*}
\big(\clip_{\ymax}\big(\theta_{\vaw, t}(X_t)^\top X_t\big) - Y_t\big)^2 & \leq \left(\half \theta_{\vaw, t}(X_t)^\top X_t + \half\clip_{\ymax}\left(\theta_{\vaw, t}(X_t)^\top X_t\right) -  Y_t \right)^2\\
        & = \big(\half \theta_{\vaw, t}(X_t)^\top X_t -  \tilde{Y}_t\big)^2~,
    \end{align*}
        where the equality is due to the definition of $\tilde{Y}_t$. Thus, by applying the above inequality and the definition of $\tilde{Y}_t$ together with \eqref{eq:clipbetter} we get
            \begin{align*}
        & \sumT \bigg(\Big(\clip_{\ymax}\big(\theta_{\vaw, t}(X_t)^\top X_t\big) - Y_t\Big)^2 - \Big(\half \clip_{\ymax}\big(\theta_{\vaw, t}(X_t)^\top X_t\big) + \half X_t^\top\theta^\star - Y_t \Big)^2\bigg) \\
        & \leq \sumT \Big(\big(\half \theta_{\vaw, t}(X_t)^\top X_t -  \tilde{Y}_t\big)^2 - \big(\half  X_t^\top\theta^\star - \tilde{Y}_t \big)^2\Big) \\
        & \leq \frac{1}{\sigma^2}\|\theta^\star\|_2^2 + \tfrac{1}{4}\max_t\{\tilde{Y}_t^2\} \sumT X_t^\top \Big(\sumt \tfrac{1}{4} X_s X_s^\top + \frac{1}{\sigma^2 }I\Big)^{-1}X_t~,
    \end{align*}
        where the last inequality is due to the regret guarantee of the Vovk-Azoury-Warmuth forecaster, see Section~4 in \cite{orabona2015generalized}. 
    
    The expression $\tfrac{1}{4}\max_t\{\tilde{Y}_t^2\}\sumT X_t^\top \Big(\sumt \tfrac{1}{4} X_s X_s^\top + \frac{1}{\sigma^2 }I\Big)^{-1}X_t$ can be bounded using standard methods, as seen on pages $318-320$ in \cite{cesa2006prediction} or in the proof of Corollary 7 in \cite{vanderhoeven2018many}. Furthermore, since $\max_t\{\tilde{Y}_t^2\} \leq 3 \ymax^2$, we have
        \begin{align*}
        \tfrac{1}{4}\max_t\{\tilde{Y}_t^2\}\sumT X_t^\top \Big(\sumt \tfrac{1}{4}X_s X_s^\top + \frac{1}{\sigma^2 }I\Big)^{-1}X_t & \leq 3\ymax^2 d\log\left(1 + \frac{T r^2\sigma^2}{4 d}\right)~.
    \end{align*}
        By utilizing $\sigma^2 = \frac{b^2}{d \ymax^2}$, we derive that $R_T \leq 4\ymax^2 d\log(1 + T \frac{b^2r^2}{4d^2\ymax^2})$. Incorporating this into equation~\eqref{eq:VAWhp1} finalizes the proof.
       \end{proof}
Proposition \ref{prop:efficientvaw2} provides a computationally efficient estimator, but its excess risk bound is weaker in terms of the dependence on $r$ compared to Proposition \ref{prop:whpvaw}. Nevertheless, the bound of Proposition \ref{prop:efficientvaw2} still shows a significant improvement over the lower bound for least squares shown in \cite{vavskevivcius2020suboptimality}. Specifically, in the setup of Proposition \ref{prop:efficientvaw2} there is a distribution with $\ymax=r=1$ and $b$ proportional to $\sqrt{d}$ such that
\[
\E\left[\E_{X, Y} \Big[\left(\widehat{\theta}_{\operatorname{ERM}}^\top X - Y\right)^2\Big] - \inf\limits_{\theta \in \Theta_b}\E_{X, Y}\left[(\theta^\top X - Y)^2\right]\right] \gtrsim \frac{d^{3/2}}{T}~,
\]
whenever $T \gtrsim d^3\log d$. Here the external expectation is taken with respect to $(X_i, Y_i)_{i = 1}^{T}$. For the same distribution, the upper bound of Proposition \ref{prop:efficientvaw2} can be written as
\[
\E\big[(\ybar_T(X) - Y)^2\big] - \inf_{\theta \in \ball} \E\big[(X^\top\theta - Y)^2\big] \lesssim \frac{d \log(T/d) + d\log(1/\delta)}{T}~.
\]
The later bound shows an improved dependence on the dimension. 

\section{Computational complexity and additional remarks}
\label{sec:compute}
Existing high-probability risk bounds for {improper} linear and logistic regression, see \cite{mourtada2021distribution} and \cite{vijaykumar2021localization} respectively, are computationally intractable or have exponential computational complexity in terms of the dimension. In contrast, our second algorithm for linear regression can be implemented in $O(d^2 T)$ runtime. A small {variation} of our algorithm for logistic regression can also be implemented efficiently. By replacing the Gaussian prior with a uniform prior over the unit ball we can apply the analysis presented in \cite[Appendix B]{foster2018logistic} to obtain the same bound with a {polynomial} algorithm. Specifically, the authors of \cite{foster2018logistic} develop a randomized implementation of their algorithm with polynomial runtime in the relevant parameters, which, with some minor changes, can also lead to an implementation of our algorithm. 

On the other hand, several efficient algorithms exist for logistic regression that are computationally efficient \cite{mourtada2019improper,jezequel2020efficient, jezequel2021mixability, agarwal2022efficient}. However, neither of these algorithms has been shown to guarantee high-probability excess risk bounds or to achieve a logarithmic dependence on the parameters. Proposition \ref{prop:efficientvaw2} plays a similar intermediate role in the context of these results for improper learners in linear regression. Our algorithm is computationally efficient, implies a high-probability excess risk upper bound, and outperforms constrained linear least squares. However, its dependence on the parameters may not be optimal.

\paragraph{Acknowledgments.}{
This work was partially done while DvdH was at the University of Milan partially supported by the MIUR PRIN grant Algorithms, Games, and Digital Markets (ALGADIMAR) and partially done while DvdH was at the University of Amsterdam supported by Netherlands Organization for Scientific Research (NWO), grant number VI.Vidi.192.095. NCB was partially supported by the EU Horizon 2020 ICT-48 research and innovation action under grant agreement 951847, project ELISE (European Learning and Intelligent Systems Excellence) and by the FAIR (Future Artificial Intelligence Research) project, funded by the NextGenerationEU program within the PNRR-PE-AI scheme. NZh is grateful to Jaouad Mourtada for several insightful discussions about the topic.
}

\vskip 0.2in
\DeclareRobustCommand{\VAN}[3]{#3}
{\footnotesize
\bibliography{sample.bib}
\bibliographystyle{abbrv}}
\DeclareRobustCommand{\VAN}[3]{#2}

\appendix
\section{Auxiliary lemmas}\label{app:concentrarion}
We need the following concentration inequality for martingales whose proof can be found in \cite[Theorem 1]{beygelzimer2011contextual}.
\begin{lemma}[A version of Freedman's inequality]\label{lem:bernie}
Let $X_1, \ldots, X_T$ be a martingale
difference sequence adapted to a filtration $(\mathcal{F}_i)_{i \le T}$. That is, in particular, $\E_{t-1}[X_t] = 0$.
Suppose that $|X_t| \leq R$ almost surely. Then for any $\delta \in (0,1), \lambda \in [0, 1/R]$, with probability at least $1-\delta$, it holds that
\begin{equation*}
\label{eq:firstfreedmanineq} 
    \sumT X_t \leq \lambda(e-2)\sumT \E_{t-1}[X_t^2] + \frac{\ln(1/\delta)}{\lambda}~.
\end{equation*}
\end{lemma}

We also use the following result.

\begin{lemma}\label{lem:surrogateexpconcave}
Suppose that $h:\domainw \mapsto \mathbb{R}$ is $\alpha$-exp concave. Then for $x, y \in \domainw$ the function $\tilde{h}(x) = h\big(\half x + \half y\big)$ is $ \alpha$-exp-concave.
\end{lemma}
\begin{proof}
We have that 
\begin{align*}
    \alpha (\tilde{h}'(x))^2 & = \frac{\alpha}{4}\big( h'(\half x + \half y)\big)^2 \leq \frac{1}{4}h''(\half x + \half y) = \tilde{h}''(x)~,
\end{align*}
where the inequality is due to the exp-concavity assumption on $h$. Thus, we have that $\alpha (\tilde{h}'(x))^2 \leq \tilde{h}''(x)$ and therefore we may conclude that $\tilde{h}$ is $\alpha$-exp concave. 
\end{proof}

\section{Exponential weights}\label{app:exponential weights}
Let
\begin{align}\label{eq:ewdistr}
    dP_{t+1}(f) = \frac{e^{-\alpha \sumt \tilde{\ell}_s(f)}dP_1(f)}{\int e^{-\alpha \sumt \tilde{\ell}_s(f)}dP_1(f)}~,
\end{align}
where $P_1$ is a prior distribution over $\Fset$, $\tilde{\ell}_t(f) = \ell(\half f(X_t) + \half \hat f_{t}(X_t), Y_t)$, and $\hat f_t = \E_{f \sim P_t}[f]$. 
This is known as the exponential weights algorithm on losses $\tilde{\ell}_1,\ldots,\tilde{\ell}_{t}$. 
\begin{lemma}\label{lem:EWshiftloss}
Suppose that $\ell:\domainw \times \Yset \mapsto [0, m]$ is $\alpha$-exp-concave in its first argument. Then, with $\hat{f} = \E_{f \sim P_t}[f] $, and with $P_t$ as defined in equation~\eqref{eq:ewdistr} for any prior distribution $P_1$ over $\Fset$,
\begin{align*}
    \sumT\left(\ell(\hat f_{t}(X_t), Y_t) -  \E_{f\sim Q}\big[\ell(\half f(X_t) + \half \hat f_{t}(X_t), Y_t)\big]\right) \leq \frac{\KL(Q\|P_1)}{\alpha}~.
\end{align*}
\end{lemma}
\begin{proof}

Since the losses $\tilde{\ell}_t$ are convex, a standard computation as in \cite[Lemma~1]{vanderhoeven2018many} shows that for any distribution $Q$ over $\Fset$~,
\begin{align*}
    & \sumT \Big(\tilde{\ell}_t(\hat f_t) - \E_{f \sim Q}\big[\tilde{\ell}_t(f)\big]\Big) \\
    & \leq \frac{\KL(Q\|P_1)}{\alpha} + \sumT \Bigg(\tilde{\ell}_t(\hat f_t) + \frac{1}{\alpha} \ln\Big(\E_{f \sim P_t}\Big[e^{-\alpha \tilde{\ell}_t(f)}\Big]\Big)\Bigg) \\
    & \leq \frac{\KL(Q\|P_1)}{\alpha}~,
\end{align*}
where the second inequality is due to the fact that $\tilde{\ell}_t$ is $\alpha$-exp-concave (Lemma~\ref{lem:surrogateexpconcave}).
\end{proof}

\section{Computations of Example \ref{ex:gausslinmodel}}
\label{compexample}
To verify the bound appearing in Example \ref{ex:gausslinmodel}, we provide the following computation.
For any $\theta \in \Theta_b$, one can easily check that
\begin{align*}
&\left|\ell_{\mu, {p_0}}\left(\E_{\theta \sim P_t}\left[p(Y_t|X_t, \theta)\right]\right) -\ell_{\mu, {p_0}}(p(Y_t|X_t, \theta))\right|
\\
&= \left|\log\left(\frac{\frac{(1 - \mu)}{\sqrt{\pi}}\left(\E_{\theta \sim P_t}[e^{-(Y_t - X_t^\top \theta)^2}]\right)+\frac{\mu}{2\sqrt{\pi}}\left(e^{-(Y_t - rb)^2} + e^{-(Y_t + rb)^2}\right)}{(1 - \mu)p(Y_t|X_t^{\top}\theta) + \frac{\mu}{2\sqrt{\pi}}\left(e^{-(Y_t - rb)^2} + e^{-(Y_t + rb)^2}\right)}\right)\right|
\\
&\le \log\left(1 + \frac{2(1 - \mu)}{\mu}\right) \le \log\left(\frac{2}{\mu}\right)~.
\end{align*}
Corollary~\ref{cor:GLM} and optimization with respect to $\mu$ conclude the derivation.

\end{document}